\documentclass[10pt, conference]{IEEEtran}
\IEEEoverridecommandlockouts
\usepackage{amsmath}
\usepackage{amssymb}
\usepackage{url}
\usepackage{cite}
\usepackage{flushend}
\usepackage{makecell}
\usepackage{subfigure}
\usepackage{amsfonts,amsthm}

\usepackage{algorithmic}
\usepackage{algorithm}

\usepackage{amsthm}

\theoremstyle{definition}
\newtheorem{definition}{Definition}
\usepackage{graphicx}
\usepackage{textcomp}
\usepackage{float}
\usepackage{booktabs}
\usepackage{bm}
\usepackage{float}  
\usepackage{pifont} 
\usepackage{wasysym}

\usepackage{multirow}
\usepackage[normalem]{ulem}
\useunder{\uline}{\ul}{}
\usepackage{xcolor}
\def\BibTeX{{\rm B\kern-.05em{\sc i\kern-.025em b}\kern-.08em
    T\kern-.1667em\lower.7ex\hbox{E}\kern-.125emX}}
\allowdisplaybreaks
\newtheorem{lemma}{Lemma}
\newtheorem{theorem}{Theorem}

\usepackage{tablefootnote}

\begin{document}

\title{In-Context Source and Channel Coding}

\author{Ziqiong Wang, Tianqi Ren, Rongpeng Li, Zhifeng Zhao, and Honggang Zhang
    \thanks{Z. Wang, T. Ren and R. Li are with the College of Information Science and Electronic Engineering, Zhejiang University (email: \{wangziqiong, rentianqi, lirongpeng\}@zju.edu.cn).}
    \thanks{Z. Zhao is with Zhejiang Lab as well as the College of Information Science and Electronic Engineering, Zhejiang University (email: zhaozf@zhejianglab.com).}
    \thanks{H. Zhang is with the Macau University of Science and Technology, China (email:  honggang.zhang@ieee.org).}
}

\maketitle

\begin{abstract}
Separate Source–Channel Coding (SSCC) remains attractive for text transmission due to its modularity and compatibility with mature entropy coders and powerful channel codes. However, SSCC often suffers from a pronounced cliff effect in low Signal-to-Noise Ratio (SNR) regimes, where residual bit errors after channel decoding can catastrophically break lossless source decoding, especially for Arithmetic Coding (AC) driven by Large Language Models (LLMs). This paper proposes a receiver-side In-Context Decoding (ICD) framework that enhances SSCC robustness without modifying the transmitter. ICD leverages an Error Correction Code Transformer (ECCT) to obtain bit-wise reliability for the decoded information bits. Based on the context-consistent bitstream, ICD constructs a confidence-ranked candidate pool via reliability-guided bit flipping, samples a compact yet diverse subset of candidates, and applies an LLM-based arithmetic decoder to obtain both reconstructions and sequence-level log-likelihoods. A reliability–likelihood fusion rule then selects the final output. We further provide theoretical guarantees on the stability and convergence of the proposed sampling procedure. Extensive experiments over Additive White Gaussian Noise (AWGN) and Rayleigh fading channels demonstrate consistent gains compared with conventional SSCC baselines and representative Joint Source-Channel Coding (JSCC) schemes.
\end{abstract}


\section{Introduction}

In recent years, Semantic Communications (SemCom) have attracted growing attention~\cite{semantics4, luRethinkingModernCommunication2023,semantics2}, driven by the need to deliver task-relevant meaning rather than faithfully transporting raw symbols, and further accelerated by advances in deep learning and generative models~\cite{semantics_generative, semantics_generative2}. 
While deep learning-based Joint Source-Channel Coding (JSCC) enables end-to-end semantic-aware transceivers with robustness under finite blocklength and time-varying channels~\cite{JSCC3,JSCC7}, Separate Source-Channel Coding (SSCC) remains attractive due to its modularity and standard compliance~\cite{SSCC1}. By decoupling source coding and channel coding, SSCC can directly reuse mature compressors~\cite{source_code} and capacity-approaching channel codes~\cite{channel_code}, and in principle achieves near-error-free, bit-exact reconstruction at sufficiently high Signal-to-Noise Ratio (SNR). 
Meanwhile, motivated by the prediction-compression duality~\cite{LLM_compress_predict}, modern neural predictors, ranging from recurrent and hybrid compressors~\cite{DZip} to Large Language Model (LLM)-based schemes~\cite{LLM_coding5}, can benefit the source encoders such as Huffman coding~\cite{Huffman_coding} and Arithmetic Coding (AC)~\cite{AC_coding}, yielding improved transmission efficiency.
However, though LLM-based SSCC could bring superior performance than JSCC \cite{SSCC1}, it suffers from the severe ``cliff effect'', whereby the performance degrades sharply in low-SNR regimes. This degradation primarily arises because even a small number of residual bit errors after channel decoding may catastrophically break subsequent source decoding.

There is no doubt that to improve robustness in low-SNR regimes, SSCC can resort to enhancing the channel-coding side~\cite{LDPC,Polar_code}.
However, under finite blocklength and stringent latency constraints, closing the gap to optimal decoding remains challenging. Learning-aided decoders have been explored, including model-based designs~\cite{model_based4} and model-free approaches~\cite{model_free2}. Notably, the Error Correction Code Transformer (ECCT)~\cite{ECCT} has incorporated code-aware inductive biases via parity-check-guided masked self-attention and has achieved error-rate improvements in low-SNR conditions. Nevertheless, such decoders remain constrained by channel observations alone~\cite{ECCT_limit}, without introducing explicit external information. 

On the other hand, leveraging contextual information~\cite{context3} at the receiver offers a backward-compatible means to enhance robustness without violating existing SSCC architectures. However, directly applying in-context decoding to corrupted bitstreams is often insufficient, as early decoding errors can dominate the inference process and lead to unstable or suboptimal reconstructions. To address this limitation, a natural approach is to maintain multiple candidates during decoding~\cite{candidate_decoding}, perform backtracking~\cite{backtracking}, or introduce diversity through stochastic sampling~\cite{sampling}. However, exhaustively exploring the candidate space is computationally prohibitive when expensive LLM-based decoding \cite{SSCC1} is involved, while unconstrained sampling may generate redundant or low-quality candidates. In principle, both in-context information and channel decoding reliability can contribute to maintaining a compact yet diverse set of plausible candidates for improved decoding performance.

In this paper, we propose in-context source and channel coding for SSCC-based transmission, aiming to mitigate the cliff effect. Particularly, we calibrate a receiver-side In-Context Decoding (ICD) framework that augments a standard SSCC pipeline by coupling ECCT-assisted channel decoding with LLM-driven arithmetic source decoding in a lightweight manner. Here, ``in-context'' refers to the same fundamental conditioning principle as Natural Language Processing (NLP)-style in-context learning, where context changes the model's conditional output distribution and restricts the plausible output space. Specifically, ECCT provides a bit-wise reliability vector that identifies probabilistically corrupted positions, while in-context information, delivered through previous transmissions and/or other reliable channels, offers a strong prior to reduce ambiguity in the recovered message. Based on the channel-decoded bitstream, its reliability vector, and the contextual bits, we leverage In-Context Candidate Generator (CCG) to construct a confidence-ranked candidate pool through reliability-guided bit flipping. Afterward, to balance the computational budget and decoding accuracy, we adopt In-Context Candidate Sampler (CCS), which selects a sufficiently compact yet highly plausible subset of candidates and their sequence-level log-likelihoods for LLM-based source decoding. Finally, we adopt In-Context Likelihood Ranking (CLR) to select the most plausible reconstruction by jointly accounting for ECCT-derived reliability and linguistic plausibility. Compared to directly decoding a single channel-decoded bitstream as in a conventional ECCT-assisted LLM-AC cascade~\cite{SSCC1}, ICD treats the receiver-side bitstream estimate as an uncertain observation and performs candidate-based decoding by jointly exploiting channel-side reliability, contextual side information, and source-level likelihood. Besides the key distinctions with representative prior studies in Table~\ref{tab:comparison}, we summarize our contributions as follows.
\begin{itemize}
    \item We introduce a practical in-context source and channel coding mechanism that incorporates contextual information into the SSCC receiver and explicitly maintains context consistency via an overwrite-based step, thereby constraining the feasible message space and improving robustness without modifying the transmitter.
    \item We develop a three-stage candidate processing pipeline through leveraging the ECCT-provided bit-wise reliability. Specifically, we employ CCG to construct a confidence-ranked set of candidate bitstreams via reliability-guided bit flipping. To achieve a favorable accuracy–complexity trade-off, we further apply CCS to select a compact yet diverse subset. Finally, we use CLR to integrate the ECCT-derived reliability with the LLM decoding log-likelihood to determine the final reconstruction. Moreover, we provide theoretical guarantees on the stability and convergence of the proposed CCS module.
    \item We conduct extensive evaluations and demonstrate the superiority of ICD over conventional SSCC baselines, represented by Huffman-SSCC and ECCT-aided scheme~\cite{SSCC1}, as well as over representative JSCC schemes, including DeepSC~\cite{DeepSC}, Universal Transformer (UT)~\cite{UT}, and UT with quantization~\cite{UT_quanti}.
\end{itemize}

The remainder of the paper is organized as follows. Sec.~\ref{sec2:Prelimiaries} succinctly reviews the key components of our SSCC system. Sec.~\ref{sec3:system_model} briefly introduces the system model and formulates the problem. Sec.~\ref{sec:algorithm} presents the overview of our proposed ICD framework. In Sec.~\ref{sec4_Experiment}, we elaborate on the experimental results and discussions. Finally, Sec.~\ref{sec5_Conclusions} concludes the paper. For convenience, we list the major notations of this paper in Table~\ref{tab1:natation}.

\begin{table*}[!t]
\centering
\caption{Summary and comparison of related papers.}
\label{tab:comparison}

\begin{tabular*}{\textwidth}{m{3cm}m{4cm}m{4.5cm}m{4.5cm}}
\toprule
  References & Strengths & Limitations & Relation to ICD  \\ \hline
  JSCC-based SemCom~\cite{DeepSC,UT,UT_quanti} &  Enables robust semantic reconstruction and graceful degradation under noisy channels. & Requires end-to-end transceiver redesign, limiting SSCC compatibility and bit-exact recovery. & ICD improves low-SNR robustness in a modular manner. \\
  Source-side coding enhancement for SSCC~\cite{DZip,Huffman_coding,AC_coding}
  & Improves transmission efficiency through stronger probabilistic source modeling.
  & Still vulnerable to residual-error-induced cliff effects.
  & ICD complements efficient source coding by stabilizing LLM-AC under channel errors.\\
  Channel-side reliability enhancement for SSCC~\cite{LDPC,Polar_code,model_based4,model_free2,ECCT,SSCC1} 
  & Improves bit-level recovery while preserving the modular SSCC pipeline. 
  & Mainly relies on channel observations, leaving LLM-AC sensitive to residual bit errors. 
  & ICD extends ECCT outputs from single-stream decoding to reliability-guided candidate exploration. \\ 
  Context- and candidate-aided decoding~\cite{context3,candidate_decoding,backtracking,sampling}
  & Exploits receiver-side contextual information and multiple decoding candidates.
  & May suffer from early-error propagation, high search cost, or redundant candidates.
  & ICD enables reliability-guided, in-context candidate processing.
  \\ \hline
  This paper  
  & Jointly exploits contextual information, ECCT-derived bit-wise reliability, and LLM likelihood under a limited decoding budget. 
  & Requires contextual side information and multiple LLM decoding attempts. 
  & Provides a practical robustness enhancement for LLM-based SSCC. \\
\bottomrule

\end{tabular*}
\end{table*}

\begin{table*}[!t]
\centering
\caption{Major notations used in this paper.}
\label{tab1:natation}
\begin{tabular*}{.85\textwidth}{cl}
\toprule
  Notation & Definition \\\hline

  $\textbf{s}_{1:N_s}$, $N_s$ & Input text sequence and its length.\\
  $\textbf{t}_{1:N_t}$, $N_t$ & Token sequence produced by the LLM tokenizer and its length.\\
  $\textbf{m}, \hat{\textbf{m}}$ & Source-coded message and its channel-decoded version. \\
  $\textbf{x}_b$, $\textbf{x}_s$ & Channel-coded binary codeword and its BPSK-modulated symbol sequence. \\

  $N, K$ & Codeword length and message length of channel code.\\
  $\mathcal{D}=\{D_1,\ldots,D_\tau\}$ & Token vocabulary defined by the tokenizer and $\tau$ is the vocabulary size. \\

  $\tilde{p}(\cdot)$ & LLM-induced conditional probability used as the distribution for AC. \\ 

  $\mathbf{G}, \mathbf{H}$ & Generator matrix and parity check matrix. \\

  $h$ & Channel fading coefficient. \\
  $\textbf{z}, \hat{\textbf{z}}$ & Additive Gaussian noise and ECCT output estimating the multiplicative disturbance. \\
  $\sigma_n^2$, $\mathbf{I}$ & Noise variance and identity matrix in the Gaussian noise model. \\

  $\textbf{y}$ & Received real-valued channel output vector. \\
  $\tilde{\textbf{y}}$ & ECCT input feature vector (i.e., concatenation of magnitude and syndrome features of $\textbf{y}$). \\
  $f_k$ & Bit-flip indicator when constructing the candidate $\tilde{\textbf{m}}$.\\
  $\boldsymbol{\rho}$ & Normalized bit-wise reliability vector obtained by applying sigmoid to $\hat{\mathbf{z}}$. \\
  $\boldsymbol{\rho}_{\mathrm{m}}$ & Bit-wise reliability vector for the information bits.\\
  $\tilde{\boldsymbol{\rho}}_{\mathrm{m}}$ & Candidate-specific reliability associated with $\tilde{\textbf{m}}$.\\

  $\ell$ & Sequence-level log-likelihood of the LLM-decoded candidate.\\
  $\hat{\textbf{x}}_s, \hat{\textbf{x}}_b$ & ECCT-processed soft symbol estimate and hard-decision demodulated estimate. \\
  $\textbf{m}_{\mathrm{pre}}$ & Contextual information bit sequence. \\
  $\tilde{\textbf{m}}$ & Candidate bitstream generated by applying a specific bit-flip pattern to $\hat{\textbf {m}}$.  \\
  $\tilde{\textbf{s}}$ & Reconstructed text sequence by the LLM-based source decoder from candidate bitstream $\tilde{\textbf{m}}$.\\
  $\tilde{\textbf{s}}^\ast$ & Final reconstructed text sequence. \\
  $\mathcal{M}_{L_c}$, $L_c$ & Candidate set retained by CCG and its size. \\
  $\mathcal{M}_{L_s}$, $L_s$ & Subset sampled by CCS and its size. \\

  $\Omega$ & State space consisting of all subsets of $\mathcal{M}_{L_c}$ with fixed cardinality $L_s$.\\
  $\mathcal{S}^{(t)}$ & State of the Markov chain at iteration $t$, represented as a subset of $\mathcal{M}_{L_c}$.\\
  $E(\mathcal{\cdot})$ & Energy function evaluating the quality of subset in terms of reliability and diversity.\\
  $\pi(\mathcal{\cdot})$ & Target probability distribution over the state space $\Omega$.\\
\bottomrule
\end{tabular*}
\end{table*}


\section{Preliminaries}\label{sec2:Prelimiaries}

\subsection{LLM-based Source Coding}
\label{sec2:LLM_coding}
We consider LLM-based Arithmetic Coding (AC) \cite{AC_coding} for source coding \cite{SSCC1}. For a tokenizer vocabulary $\mathcal{D}=\{D_1,\ldots,D_\tau\}$, a pre-trained LLM can map a source sequence $\textbf{s}$ to a token sequence $\textbf{t}_{1:N_t}$ with $t_k\in\mathcal{D}$, while providing token-level conditional probabilities $\tilde{p}(t_k \mid \textbf{t}_{1:k-1})$ for generation. Using these conditionals, AC maps $\textbf{t}_{1:N_t}$ to a binary message $\textbf{m}\in\{0,1\}^{K}$. By progressively updating a probability interval by the autoregressive structure of the LLM, the source can eventually be encoded with a minimal number of bits\footnote{The standard AC interval updates and the corresponding LLM-AC decoding steps can be found in~\cite{SSCC1}.}. 

At the receiver, given a channel-decoded bitstream estimate $\hat{\textbf{m}}$ and the same LLM, an arithmetic decoder gradually locates the most appropriate interval corresponding to a token sequence $\tilde{\textbf{t}}$ and deterministically recovers the reconstructed text $\tilde{\textbf{s}}$. Beyond producing the reconstruction, the LLM also yields a sequence-level plausibility for the decoded sequence by accumulating token log-probabilities along the decoding path:
\begin{equation}
\ell(\tilde{\textbf{t}})
\triangleq
\sum_{k=1}^{|\tilde{\textbf{t}}|}
\log \tilde{p}(\tilde{t}_{k}\mid \tilde{\textbf{t}}_{1:k-1}).
\label{eq:seq_ll}
\end{equation}

For convenience, we denote the transmitter-side LLM-driven arithmetic encoding operation by $\mathcal{F}_{\mathrm{LLM}_s}(\cdot)$ and the receiver-side arithmetic decoding and likelihood evaluation operation by $\mathcal{F}_{\mathrm{LLM}_r}(\cdot)$.

\subsection{Error Correction Code Transformer}\label{2.2_ECCT}

We adopt the ECCT\footnote{Complete architectural and training details of ECCT can be found in~\cite{ECCT}.}~\cite{ECCT} as a receiver-side reliability estimation module to quantify bit-wise confidence under severe noise. Without loss of generality, assume that for a Binary Phase-Shift Keying (BPSK)-coded sequence $\textbf{x}_s\in\{\pm 1\}^{N}$, the corresponding channel output can be written as: 
\begin{equation}
\label{eq:channel}
\textbf{y}=h\textbf{x}_s+\textbf{z},\qquad \textbf{z}\sim\mathcal{N}(\textbf{0},\sigma_n^2\mathbf{I}),
\end{equation}
where $h$ is the Additive White Gaussian Noise (AWGN) or Rayleigh fading coefficient, and $\textbf{z}\sim\mathcal{N}(\textbf{0},\sigma_n^2\mathbf{I})$ represents additive Gaussian noise. Following~\cite{ECCT}, the channel observation can also be expressed in a multiplicative form $\textbf{y}=\textbf{x}_s\odot\tilde{\textbf{z}}$, where $\tilde{\textbf{z}}\in\mathbb{R}^{N}$ denotes an equivalent multiplicative disturbance. 

We consider a binary linear block code of length $N$ and dimension $K$ characterized by a parity-check matrix $\mathbf{H}\in\{0,1\}^{(N-K)\times N}$ (e.g., a Low-Density Parity-Check (LDPC) code). Then, given the real-valued channel output $\textbf{y}\in\mathbb{R}^{N}$, ECCT exploits the code constraints $\mathbf{H}$ via the binary syndrome computed from a hard-decision mapping. 
Specifically, define the syndrome vector as:
\begin{equation}
\text{syn}(\textbf{y}) \triangleq \mathbf{H}\,\text{sign\_to\_bin}(\textbf{y})\ (\mathrm{mod}\ 2)\in\{0,1\}^{N-K},
\end{equation}
where $\text{sign\_to\_bin}(\textbf{y}) \triangleq \frac{1}{2}\big(\textbf{1}-\text{sign}(\textbf{y})\big)\in\{0,1\}^{N}$ while $\text{syn}(\textbf{y})=\textbf{0}$ indicates that the hard-decision codeword satisfies all parity checks. 
To jointly incorporate soft channel observations and code-structure information, ECCT forms an augmented input feature vector:
\begin{equation}
\tilde{\textbf{y}} \triangleq \big[\,|\textbf{y}|,\ \text{syn}(\textbf{y})\,\big]\in\mathbb{R}^{2N-K}.
\end{equation}

ECCT processes $\tilde{\textbf{y}}$ 
and outputs an estimate of the disturbance:
\begin{equation}
\hat{\textbf{z}}=\mathcal{F}_{\text{ECCT}}(\tilde{\textbf{y}}),
\end{equation}
which implicitly captures position-dependent decoding uncertainty while respecting parity-check consistency. 
On this basis, ECCT computes a recovered BPSK-coded sequence by
\begin{equation}
\hat{\textbf{x}}_s=\text{sign\_to\_bin}\!\left(\textbf{y}\odot \hat{\textbf{z}}\right)\in\{0,1\}^{N}.
\label{eq:ecct_post}
\end{equation}

In this work, we further convert $\hat{\textbf{z}}$ into a bit-wise reliability vector:
\begin{equation}
\boldsymbol{\rho}=\sigma(\hat{\textbf{z}})\in[0,1]^N,
\end{equation}
where a larger $\rho_d$ indicates higher confidence on the accuracy of the $d$-th decoded bit. 
The reliability $\boldsymbol{\rho}$ is used as the common input to the subsequent ICD modules to guide candidate construction and selection. 

\section{System model and problem formulation}\label{sec3:system_model}

\subsection{System Model}

\begin{figure*}[!t]
\centering
\includegraphics[width=0.75\linewidth]{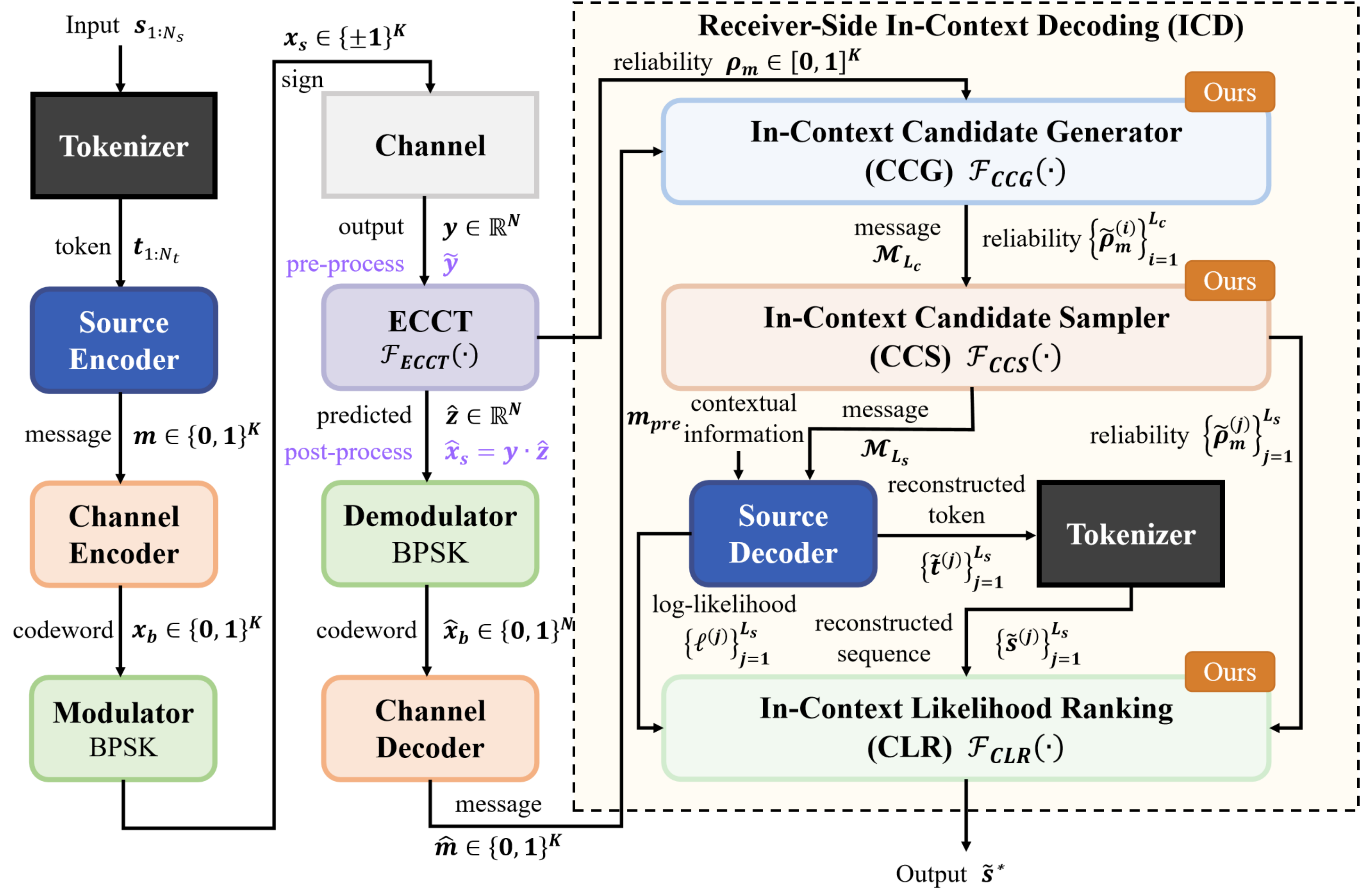}
\caption{Framework of the Proposed ICD-Aided SSCC System.}
\label{fig:system model}
\end{figure*}

We consider an SSCC-based text transmission framework in which the receiver is augmented with ECCT-assisted channel decoding and an LLM-based source decoder equipped with ICD, which comprises CCG, CCS, and CLR, in addition to conventional source and channel coding, modulation, and demodulation. Fig.~\ref{fig:system model} shows the corresponding system block diagram. 

For the input text sequence $\textbf{s}_{1:N_s}$, a lossless source encoder compresses it into a binary message:
\begin{equation}
\textbf{m}=\mathcal{F}_{\mathrm{LLM}_s}(\textbf{s}_{1:N_s})\in\{0,1\}^{K},
\end{equation}
where as mentioned in Sec. \ref{sec2:LLM_coding}, $\mathcal{F}_{\mathrm{LLM}_s}(\cdot)$ denotes a LLM-compatible source encoder that maps text into a length-$K$ bitstream through tokenization and probability-driven AC. 
The same pipeline can also be applied to other discrete sources (e.g., images) after appropriate lossless formatting or compression.

The message $\mathbf{m}$ is then protected by an $(N,K)$ LDPC channel code as:
\begin{equation}
\textbf{x}_{b}=\textbf{m}\mathbf{G}\ (\mathrm{mod}\ 2)\in\{0,1\}^{N},
\end{equation}
where $\mathbf{G}\in\{0,1\}^{K\times N}$ is the generator matrix. The corresponding parity-check matrix $\mathbf{H}\in\{0,1\}^{(N-K)\times N}$ satisfies:
\begin{equation}
\mathbf{G}\mathbf{H}^{\top}=\mathbf{0}.
\end{equation}
Notably, other error correction codes such as Polar codes~\cite{Polar_code} can be applied as well \cite{ECCT}. The binary codeword $\textbf{x}_{b}$ is modulated via BPSK into $\textbf{x}_s=1-2\textbf{x}_{b}\in\{\pm 1\}^{N}$ and transmitted over the channel, as defined in Eq. \eqref{eq:channel}.

At the receiver, ECCT takes the real-valued channel output $\textbf{y}$ and the code constraints as input, and produces a recovered BPSK-coded sequence $\hat{\textbf{x}}_s$, an estimate of disturbance $\hat{\textbf{z}}$, and a bit-wise reliability vector $\boldsymbol{\rho}\in[0,1]^N$ as described in Sec.~\ref{2.2_ECCT}. Hence, hard-decision BPSK demodulation yields the binary codeword estimate:
\begin{equation}
\hat{\textbf{x}}_{b}=\text{sign\_to\_bin}(\hat{\textbf{x}}_{s})\in\{0,1\}^{N}.
\end{equation}

In this work, we adopt a systematic channel encoder, such that the information bits occupy the first $K$ positions of the codeword.
Since ECCT operates as a channel decoder producing an estimated codeword, recovering the information bitstream reduces to extracting the first $K$ bits:
\begin{equation}
\hat{\textbf{m}}=\hat{\textbf{x}}_{b,1:K}\in\{0,1\}^{K},
\end{equation}

Similarly, since the ECCT-derived reliability $\boldsymbol{\rho}\in[0,1]^N$ is defined at the codeword-bit level, we take its subvector aligned with the information positions:
\begin{equation}
\boldsymbol{\rho}_{\text{m}}=\boldsymbol{\rho}_{1:K}\in[0,1]^K .
\end{equation}

Given the channel decoding result $\hat{\textbf{m}}$, as mentioned in Sec. \ref{sec2:LLM_coding}, an LLM-based source decoder $\mathcal{F}_{\text{LLM}_r}(\cdot)$ can eventually produce a reconstructed sequence $\tilde{\textbf{s}}$, along with its linguistic log-likelihood $\ell$: 
\begin{equation}
\left(\tilde{\textbf{s}},\ell\right)
=\mathcal{F}_{\text{LLM}_r}\!\left(\hat{\textbf{m}}\right).
\end{equation}

\subsection{Problem Formulation}

Conventional SSCC systems often suffer from a \emph{cliff effect} in low-SNR scenarios, where crossing below a specific SNR threshold triggers an abrupt rise in the bit error rate and a severe drop in reconstruction fidelity, as illustrated by the second black line in Fig. \ref{fig:overall_performance_awgn} and Fig. \ref{fig:overall_performance_rayleigh}. The primary reason lies in the fact that an incorrect early bit can shift the LLM-based decoder to a wrong sub-interval for arithmetic decoding, which then perturbs subsequent token boundary decisions and triggers error propagation across many tokens. Conceptually, LLM-driven AC performs a strictly sequential interval-localization process. Starting from the initial interval $[0,1)$ and the Beginning-Of-Sequence (BOS) condition, the LLM predicts the distribution of the first token and partitions the current interval into token-specific sub-intervals. The received bitstream selects a sub-interval and the associated token is appended to the prefix for subsequent LLM prediction. Since later interval partitions depend on the decoded prefix, a residual bit error may shift the binary value across a sub-interval boundary and induce a wrong token decision. This erroneous prefix then causes the decoder trajectory to deviate from the encoder-side partitioning process. Consequently, even a small number of residual bit errors after channel decoding may catastrophically break lossless source decoding. 

To improve reliability under such unfavorable conditions, we consider receiver-side enhancements that utilize additional contextual information $\textbf{m}_{\text{pre}}$, which has been obtained through previous transmissions and/or other reliable channels. In practical wireless deployments, such contextual information can be obtained through temporal caching or heterogeneous transmission. For continuous transmission sessions, $\textbf{m}_{\text{pre}}$ may consist of previously transmitted segments that have already been verified by the receiver's Cyclic Redundancy Check (CRC) and stored in a local cache. Alternatively, in mission-critical scenarios, critical metadata or contextual information can be transmitted via an Ultra-Reliable Low-Latency Communication (URLLC)-style reliable channel~\cite{context_channel1, context_channel2,context_channel3} to ensure highly reliable delivery. To quantify the overhead, we define the raw-text-level context ratio as
\begin{equation}
\eta_{\text{raw}}=\frac{N_{\text{pre}}}{N_s},
\end{equation}
where $N_{\text{pre}}$ denotes the number of words used as contextual information and $N_s$ is the length of the original text sequence. More clearly, we argue that the in-context information $\textbf{m}_{\text{pre}}$ can provide a strong probabilistic prior that helps combat the arising severe noise. In other words, if we can construct multiple candidate bitstreams $\tilde{\textbf{m}}$ from the channel decoding result $\hat{\textbf{m}}$, it will benefit the receiver by evaluating the contextual plausibility (i.e., log-likelihood) $\ell$ of one bit stream and determining the superior one. To support such an intuition, a confidence-ranked candidate set shall be computed from the available contextual information: 
\begin{equation}
\left(\mathcal{M}_{L_c},\{\tilde{\boldsymbol{\rho}}_{\text{m}}^{(i)}\}_{i=1}^{L_c}\right)
= \mathcal{F}_{\text{CCG}}\!\left(\hat{\textbf{m}},\boldsymbol{\rho}_{\text{m}};L_c\right),
\end{equation}
where $\mathcal{F}_{\text{CCG}}(\cdot)$ denotes a CCG function, while $\mathcal{M}_{L_c}=\{\tilde{\textbf{m}}^{(i)}\}_{i=1}^{L_c}$ denotes the top-$L_c$ candidate bitstreams retained according to the ECCT-derived bit-wise reliability. 

However, if all candidates go through the LLM, it will incur substantial computational cost. To better balance cost and diversity, it becomes appealing to design a sampling module $\mathcal{F}_{\text{CCS}}(\cdot)$ that selects $L_s$ candidates from $\mathcal{M}_{L_c}$:
\begin{equation}
\begin{aligned}
\left(\mathcal{M}_{L_s},\{\tilde{\boldsymbol{\rho}}_{\text{m}}^{(j)}\}_{j=1}^{L_s}\right)
&= \mathcal{F}_{\text{CCS}}\!\left(\mathcal{M}_{L_c},\{\tilde{\boldsymbol{\rho}}_{\text{m}}^{(i)}\}_{i=1}^{L_c};L_s\right), \\
L_s &\le L_c - 2,
\end{aligned}
\end{equation}
with $\mathcal{M}_{L_s}=\{\tilde{\textbf{m}}^{(j)}\}_{j=1}^{L_s}$. 
Each selected candidate is then fed into the LLM-based source decoder $\mathcal{F}_{\text{LLM}_r}(\cdot)$, producing a reconstructed sequence and its log-likelihood:
\begin{equation}
\left(\tilde{\textbf{s}}^{(j)},\ell^{(j)}\right)
=\mathcal{F}_{\text{LLM}_r}\!\left(\tilde{\textbf{m}}^{(j)};\textbf{m}_{\text{pre}}\right),\quad j=1,\ldots,L_s.
\label{eq:llm_decode}
\end{equation}

Finally, the ECCT-derived bit-wise reliability and linguistic log-likelihood can undergo a CLR $\mathcal{F}_{\text{CLR}}$ for output the final reconstruction:
\begin{equation}
\tilde{\textbf{s}}^{\ast}
=\mathcal{F}_{\text{CLR}}\!\left(\{\tilde{\textbf{s}}^{(j)}\}_{j=1}^{L_s},\{\ell^{(j)}\}_{j=1}^{L_s},\{\tilde{\boldsymbol{\rho}}_{\text{m}}^{(j)}\}_{j=1}^{L_s}\right),
\end{equation}
where $\mathcal{F}_{\text{CLR}}(\cdot)$ implements a generic confidence--likelihood fusion rule and returns the optimal candidate $\tilde{\textbf{s}}^{\ast}$ as the decoding result.

For notational convenience, we denote the decoder by $\Phi(\cdot) =\mathcal{F}_{\text{CLR}} \circ\mathcal{F}_{\text{LLM}_r} \circ  \mathcal{F}_{\text{CCS}} \circ   \mathcal{F}_{\text{CCG}} (\cdot)$. 
Our goal is to design a computationally efficient source and channel decoder $\Phi(\cdot)$ that maps the ECCT-assisted decoding output $\boldsymbol{\rho}_m$, the channel decoding result $\hat{\textbf{m}}$, and the contextual side information $\mathbf{m}_{\text{pre}}$ to the final reconstruction $\hat{\mathbf{s}}^{\ast}$: 
\begin{equation}
\begin{aligned}
&\min\  \mathbb{E}\big[ d(\textbf{s}_{1:N_s},\tilde{\textbf{s}}^\ast) \big] \\
\text{s.t.}\quad 
\mathrm{Cost}\! & \left(F_{\text{CCS}}\!\circ F_{\text{LLM}_r}\right)\le \mathcal{B},\ \tilde{\textbf{s}}^\ast =\Phi(\hat{\textbf{m}},\boldsymbol{\rho}_{\mathrm{m}},\textbf{m}_{\text{pre}}),
\label{eq:opt_problem}
\end{aligned}
\end{equation}
where $d(\cdot,\cdot)$ measures the reconstruction mismatch, and $\mathrm{Cost}\!\left(F_{\text{CCS}}\!\circ F_{\text{LLM}_r}\right)$ denotes the computational cost of candidate sampling followed by LLM-based arithmetic source decoding. Specifically, we measure it as
\begin{equation}
\mathrm{Cost}(F_{\text{CCS}}\circ F_{\text{LLM}_r}) = C_{\text{CCS}} + L_s C_{\text{LLM}_r},
\end{equation}

\noindent where $C_{\text{CCS}}$ denotes the cost of the CCS module, $L_s$ is the number of selected candidate bitstreams, and $C_{\text{LLM}_r}$ is the average cost of decoding one candidate and computing its sequence-level log-likelihood. Since $C_{\text{LLM}_r}$ dominates the lightweight candidate sampling overhead, the budget constraint in Eq.~\eqref{eq:opt_problem} can be equivalently interpreted as limiting the number of LLM-based decoding attempts under a given computational budget $\mathcal{B}$. We leave the design of the modules in ICD in Sec. \ref{sec:algorithm}.

\section{The Framework of ICD} \label{sec:algorithm}

\begin{figure*}[t]
    \centering
    \includegraphics[width=\linewidth]{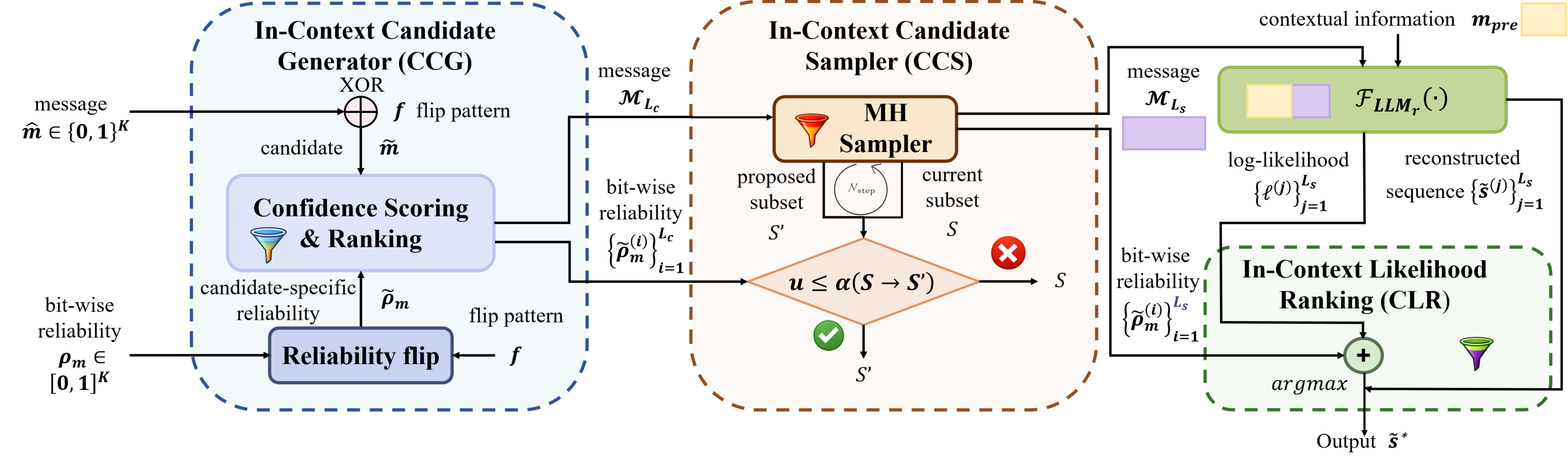}
    \vspace{-1cm}
    \caption{Framework of ICD.}
    \label{fig:CD}
\end{figure*}

Fig.~\ref{fig:CD} illustrates the overall framework of the proposed ICD.

\subsection{In-Context Candidate Generator} \label{subsec:ccg}

The role of CCG is to produce a confidence-ranked candidate set for the information-bit sequence.
Given the channel decoding output $\hat{\textbf{m}}\in\{0,1\}^{K}$,
we construct a candidate set
$\mathcal{M}\triangleq\{\tilde{\textbf{m}}^{(d)}\}_{d=1}^{2^{K}}$
by enumerating all possible bit-flip patterns. 
Specifically, let $\textbf{f}^{(d)}=[f^{(d)}_{1},\ldots,f^{(d)}_{K}]\in\{0,1\}^{K}$ denote the $d$-th flip pattern, where $f^{(d)}_{p}=1$ indicates flipping the $p$-th bit and $f^{(d)}_{p}=0$ otherwise, with $d=1,\ldots,2^{K}$. 
Given $\textbf{f}^{(d)}$, the corresponding candidate bitstream
$\tilde{\textbf{m}}^{(d)}\in\mathcal{M}$ is constructed entry-wise as:
\begin{equation}
\tilde{m}^{(d)}_{p}
=
\hat{m}_{p}\oplus f^{(d)}_{p},\quad p=1,\ldots,K,
\label{eq:ccg_candidate_construct}
\end{equation}
where $\oplus$ denotes the XOR operation.

Given the ECCT-derived bit-wise reliability vector $\boldsymbol{\rho}_{\mathrm m}=[\rho_{\mathrm m,1},\ldots,\rho_{\mathrm m,K}]\in[0,1]^{K}$, we define the candidate-specific reliability as $\tilde{\boldsymbol{\rho}}_{\mathrm m}^{(d)}=[\tilde{\rho}_{\mathrm m,1}^{(d)},\ldots,\tilde{\rho}_{\mathrm m,K}^{(d)}]$, whose $p$-th entry under candidate $\tilde{\mathbf{m}}^{(d)}$ can be computed as:
\begin{equation}
\tilde{\rho}^{(d)}_{\mathrm m,p}
=
(1-f^{(d)}_{p})\,\rho_{\mathrm m,p}
+
f^{(d)}_{p}\,(1-\rho_{\mathrm m,p}),
\quad p=1,\ldots,K.
\label{eq:ccg_reliability_flip}
\end{equation}
Accordingly, the aggregate confidence score of candidate $\tilde{\mathbf{m}}^{(d)}$ is defined as:
\begin{equation}
\mathcal{F}_{\mathrm{Conf}}\!\left(\tilde{\mathbf{m}}^{(d)}\right)
=
\sum_{p=1}^{K}\tilde{\rho}^{(d)}_{\mathrm m,p},
\quad d=1,\ldots,2^{K}.
\label{eq:ccg_conf_score}
\end{equation}

Finally, CCG retains the top-$L_c$ candidates with the largest confidence scores, yielding a confidence-ranked candidate set:
\begin{equation}
\mathcal{M}_{L_c}
=
\Big\{\tilde{\mathbf{m}}^{(i)}\ \Big|\ i \in 
\operatorname*{arg\,top}_{L_c, 1\le d \le2^K}
\mathcal{F}_{\mathrm{Conf}}\!\big(\tilde{\mathbf{m}}^{(d)}\big)
\Big\}\quad i=1,\ldots,L_{c}.
\label{eq:ccg_topLc}
\end{equation}

\subsection{In-Context Candidate Sampler}\label{subsec:dpcs}

\begin{figure*}[!t]
\centering
\includegraphics[width=0.6\linewidth]{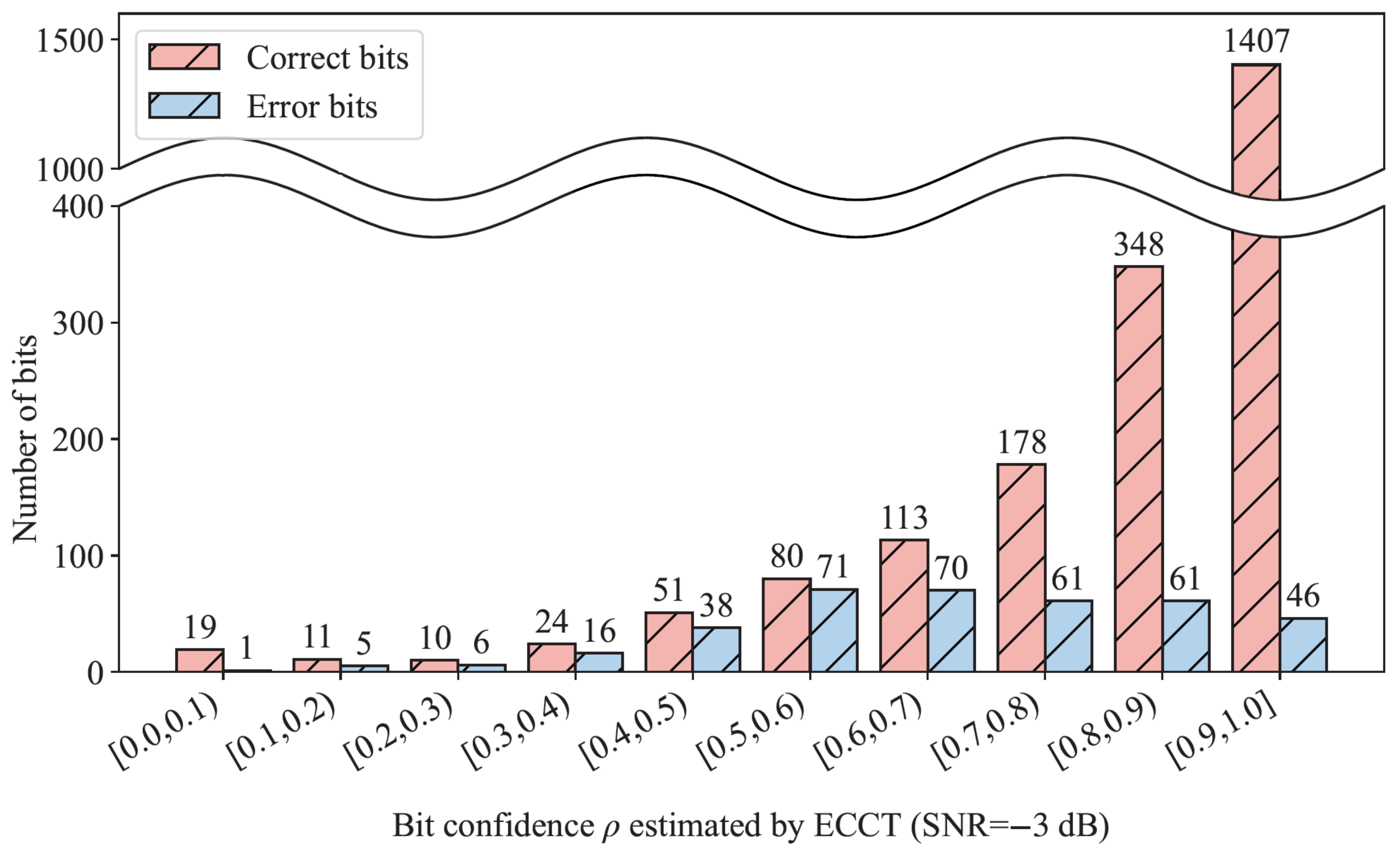}
\caption{Distribution of correct and erroneous bits across confidence levels.}
\label{fig:bar_compare}
\end{figure*}

CCG leverages the ECCT-derived bit reliability to score candidate bitstreams and efficiently prune an exponentially large flip space, resulting in a confidence-ranked set of top-$L_c$ candidates. While this confidence-guided pruning is essential for computational tractability, the resulting ranking is inevitably affected by reliability miscalibration, as evidenced by Fig.~\ref{fig:bar_compare}, which shows that a non-negligible number of erroneous bits persist even in high-confidence regimes. Moreover, candidates with high confidence scores are often generated by flipping similar subsets of low-reliability bits, leading to strong correlations among top-ranked candidates and consequently small pairwise Hamming distances. 

As a consequence, directly truncating the ranked list to the top-$L_s$ candidates would produce a highly concentrated and redundant subset, severely limiting the effective exploration space available to the downstream LLM-based decoder. To address this issue, CCS performs diversity-preserving subset selection by explicitly balancing aggregate confidence and inter-candidate Hamming diversity, so that with a limited number of LLM decoding attempts $L_s$, the selected candidates span multiple plausible error patterns and substantially increase the probability of including a decodable or correct bitstream.

This subset-level trade-off naturally leads to a large search space. For typical values of $L_c$ and $L_s$, the number of feasible subsets scales as $\binom{L_c}{L_s}$, making exhaustive evaluation costly in practice. This motivates the use of approximate methods that can efficiently explore the subset space without incurring prohibitive computational overhead. To this end, CCS adopts a sampling-based approach inspired by
Metropolis-Hastings (MH)~\cite{MH1,MH2}, which enables efficient exploration of the subset space
with a controllable computational budget while naturally introducing randomness
to promote diversity.

Specifically, CCS formulates the subset selection problem as sampling from a structured probability distribution defined over subsets of $\mathcal{M}_{L_c}$. To this end, we construct a Markov chain $\{\mathcal{S}^{(t)}\}_{t\ge 0}$ on the state space $\Omega
\triangleq
\big\{\mathcal{S}\subseteq\mathcal{M}_{L_c} \,\big|\, |\mathcal{S}|=L_s\big\}$,
where each state $\mathcal{S}^{(t)}\in\Omega$ is a candidate subset of fixed cardinality $L_s$
to be forwarded to the subsequent LLM-based decoding stage. 
Subsequently, we define the target distribution over $\Omega$ as:
\begin{equation}
\pi(\mathcal{S}) = \exp\big(-\beta E(\mathcal{S})\big), \quad \mathcal{S}\in\Omega,
\label{eq:pi_distribution}
\end{equation}
where $\beta>0$ controls the sharpness and the energy $E(\mathcal{S})$ jointly captures
ECCT-derived reliability and inter-candidate diversity, given by
\begin{equation}
E(\mathcal{S})
=
-
\sum_{\tilde{\textbf{m}}\in\mathcal{S}} \mathcal{F}_{\mathrm{Conf}}\!\left(\tilde{\textbf{m}}\right)
\;-\;
\lambda
\sum_{\substack{\tilde{\textbf{m}},\,\tilde{\textbf{m}}'\in\mathcal{S}\\ \tilde{\textbf{m}}\neq \tilde{\textbf{m}}'}}
\mathcal{F}_{\mathrm{Hamming}}\!\left(\tilde{\textbf{m}},\tilde{\textbf{m}}'\right),
\label{eq:dpcs_energy}
\end{equation}
where $\mathcal{F}_{\mathrm{Conf}}(\cdot)$ denotes the ECCT-derived aggregate confidence while $\mathcal{F}_{\mathrm{Hamming}}(\cdot)$ measures pairwise Hamming distance, and $\lambda\ge 0$ balances reliability preservation and diversity encouragement. This energy function is explicitly tailored to the residual-error characteristics observed after ECCT decoding. While the confidence term $\mathcal{F}_{\mathrm{Conf}}(\cdot)$ ensures alignment with the ECCT-derived bit-wise reliability, the Hamming-diversity term $\mathcal{F}_{\mathrm{Hamming}}(\cdot)$ discourages selecting highly similar flip patterns. As a result, the induced target distribution favors reliable yet diverse candidate subsets, allowing CCS to allocate the limited LLM decoding budget across multiple plausible residual-error patterns.

Given the current state $\mathcal{S}\in\Omega$, we generate a proposal state $\mathcal{S}'\in\Omega$
via a single replacement move $\mathcal{S}'=(\mathcal{S}\setminus \{\tilde{\textbf{m}}\})\cup \{\tilde{\textbf{m}}''\}$, by changing a uniformly sampled $\tilde{\textbf{m}} \in \mathcal{S}$ to  $\tilde{\textbf{m}}'' \in (\mathcal{M}_{L_c}\setminus\mathcal{S})$, where $\setminus$ denotes the set difference operator. 
This defines the proposal distribution:
\begin{equation}
q(\mathcal S'|\mathcal S)
=
\frac{1}{L_s\,(L_c-L_s)},
\end{equation}
and $q(\mathcal S'|\mathcal S)=q(\mathcal S|\mathcal S')$ holds due to the symmetry of the replacement move.

The Markov chain transition kernel $P$ is then induced by the MH accept-reject rule:
\begin{equation}
P(\mathcal{S}'|\mathcal{S})
=
q(\mathcal{S}'|\mathcal{S})\,\alpha(\mathcal{S}\to\mathcal{S}'),
\qquad \mathcal{S}'\neq \mathcal{S},
\end{equation}
with acceptance probability:
\begin{align}
\alpha(\mathcal{S}\rightarrow\mathcal{S}')
&=\min\left\{1,\ \frac{\pi(\mathcal S')q(\mathcal S\mid\mathcal S')}{\pi(\mathcal S)q(\mathcal S'\mid\mathcal S)}\right\} \nonumber\\
&=\min\left\{1,\ \frac{\pi(\mathcal S')}{\pi(\mathcal S)}\right\} \nonumber \\
&=
\min\!\left\{
1,\;
\exp\big(-\beta[E(\mathcal{S}')-E(\mathcal{S})]\big)
\right\} .
\label{eq:mh_acceptance_prob}
\end{align}
A uniform random number $u\sim\mathcal{U}(0,1)$ is then drawn, and the proposal is accepted if $u\le \alpha(\mathcal{S}\rightarrow\mathcal{S}')$. Otherwise, the chain remains at $\mathcal{S}$. For completeness, the self-transition probability is:
\begin{equation}
P(\mathcal{S}|\mathcal{S})
=
1-\sum_{\mathcal{S}'\in\Omega,\ \mathcal{S}'\neq\mathcal{S}} P(\mathcal{S}'|\mathcal{S}).
\label{eq:dpcs_self_transition}
\end{equation}

After a fixed number of MH iterations $N_\text{step}$, CCS outputs a sampled subset $\mathcal{M}_{L_s}$ of size $L_s$, which exhibits both high aggregate confidence and sufficient diversity across candidate bitstreams. The computational overhead of CCS is mainly determined by the sampled subset size $L_s$. In each sampling update, CCS replaces one candidate in the current subset and evaluates the corresponding change in the subset-level energy. Accordingly, the effective CCS overhead scales approximately as $\mathcal{O}(L_sK)$ per update. This overhead is lightweight compared with the LLM-based arithmetic source decoding applied to the $L_s$ sampled candidates.
This sampling-enhanced selection avoids excessive concentration on highly similar candidates that arise from correlated low-confidence bit flips, while maintaining tractable computational complexity. 

\begin{definition}[Stationarity under Detailed Balance~\cite{levin2017markov}]
\label{lem:detailed_balance_stationary}
Let $\{\mathcal S^{(t)}\}_{t\ge 0}$ be a Markov chain on a finite state space $\Omega$ with transition kernel $P$.
If there exists a distribution $\pi$ such that the detailed balance condition holds
\begin{equation}
\pi(\mathcal S)\,P(\mathcal{S}'|\mathcal{S})
=
\pi(\mathcal S')\,P(\mathcal{S}|\mathcal{S}'),
\quad \forall \mathcal S,\mathcal S'\in \Omega,
\end{equation}
then $\pi$ is a stationary distribution of the Markov chain.
\end{definition}

\begin{theorem}[Properties of CCS]
\label{thm:ccs_mh_conditions}
Consider the Markov chain $\{\mathcal S^{(t)}\}_{t\ge 0}$ induced by the CCS sampler on the state space $\Omega$, it satisfies the following properties: (1) The state space $\Omega$ is finite. (2) The transition kernel $P$ meets the detailed balance condition with respect to the target distribution $\pi$. (3) The Markov chain is irreducible and aperiodic on $\Omega$.
\end{theorem}

The proof follows the standard finite-state Markov-chain argument, and is omitted here.

\begin{lemma}[Ergodicity and Convergence~\cite{Norris1998MarkovChain}]
\label{lem:ergodic_convergence}
Let $\{\mathcal S^{(t)}\}_{t\ge 0}$ be a Markov chain on a finite state space with stationary distribution $\pi$.
If the chain is irreducible and aperiodic, then $\pi$ is unique and the distribution of $\mathcal S^{(t)}$ converges to $\pi$ in total variation distance, regardless of the initial state.
\end{lemma}

\begin{theorem}[Stationarity and Convergence of CCS]
\label{thm:dpcs_convergence}
Consider the Markov chain $\{\mathcal S^{(t)}\}_{t\ge 0}$ induced by the CCS sampler.
 $\pi$ is the unique stationary distribution, and for any initial state $\mathcal S^{(0)}$, the distribution of $\mathcal S^{(t)}$ converges to $\pi$ in total variation distance, that is
\begin{equation}
\lim_{t\to\infty}
\big\|
P^t(\mathcal S^{(0)},\cdot)-\pi(\cdot)
\big\|_{\mathrm{TV}}
=
0.
\end{equation}
\end{theorem}

\begin{proof}
By Theorem~\ref{thm:ccs_mh_conditions}, the CCS chain operates on a finite state space and satisfies detailed balance with respect to $\pi$.
Hence, by Definition~\ref{lem:detailed_balance_stationary}, $\pi$ is a stationary distribution of the chain.
Moreover, Theorem~\ref{thm:ccs_mh_conditions} also establishes that the chain is irreducible and aperiodic.
Therefore, by Lemma~\ref{lem:ergodic_convergence}, the stationary distribution $\pi$ is unique and the distribution of $\mathcal S^{(t)}$ converges to $\pi$ in total variation distance for any initial state $\mathcal S^{(0)}$.
\end{proof}

\textit{Remark}. Given this CCS-specific target distribution in Eq.~\eqref{eq:pi_distribution}, the theoretical derivations above validate the effectiveness of the proposed sampler in the present decoding setting. Theorem~\ref{thm:ccs_mh_conditions} verifies the standard finite-state MH conditions for the CCS-induced chain, while Theorem~\ref{thm:dpcs_convergence}, together with Lemma~\ref{lem:ergodic_convergence}, establishes convergence to the unique stationary distribution induced by the ECCT-derived reliability-diversity energy.

\subsection{In-Context Likelihood Ranking}
For each retained bitstream candidate $\tilde{\textbf{m}}^{(j)}\in\mathcal{M}_{L_s}$, CLR integrates ECCT-derived reliability with source-level linguistic plausibility to select the final reconstruction. Formally, CLR computes a fused score for each candidate by combining the ECCT-derived reliability profile $\tilde{\boldsymbol{\rho}}_{\mathrm{m}}^{(j)}$ and the LLM log-likelihood $\ell^{(j)}$, and outputs:

\begin{equation}
\tilde{\mathbf{s}}^{\ast}
=\tilde{\mathbf{s}}^{(j^\ast)},
j^\ast
=\arg\max_{1\le j\le L_s}
\Bigl\{
\mathcal{F}_{\mathrm{Conf}}\!\left(\tilde{\mathbf{m}}^{(j)}\right)
+\alpha\,\ell^{(j)}
\Bigr\}.
\label{eq:rlfd_explicit}
\end{equation}

In summary, the pseudocode of the proposed receiver-side ICD procedure is provided in Algorithm \ref{alg1}.

\begin{algorithm}[t]
\caption{The Procedure of ICD.}
\label{alg1}
\begin{algorithmic}[1]
    \REQUIRE Channel-decoded bitstream $\hat{\textbf m}$, ECCT-derived reliability $\boldsymbol{\rho}_{\mathrm m}$, contextual information $\textbf m_{\mathrm{pre}}$, candidate pool size $L_c$, sampled subset size $L_s$, MH steps $N_{\mathrm{step}}$, temperature $\beta$, diversity weight $\lambda$, fusion weight $\alpha$;
    \ENSURE Final reconstruction $\tilde{\mathbf s}^\ast$;
    \STATE \textbf{Stage I: CCG --- In-Context Candidate Generator}
    \STATE $\mathcal M \leftarrow$ constructs candidates $\tilde{\textbf m}$ by Eq.~\eqref{eq:ccg_candidate_construct}.
    \STATE $F_{\mathrm{Conf}}(\tilde{\textbf m}) \leftarrow$ computes confidence scores by Eq.~\eqref{eq:ccg_reliability_flip}--Eq.~\eqref{eq:ccg_conf_score};
    \STATE $\mathcal M_{L_c} \leftarrow$ retains top-$L_c$ candidates by Eq.~\eqref{eq:ccg_topLc};
    \STATE \textbf{Stage II: CCS --- In-Context Candidate Sampler}
    \STATE $\mathcal S^{(0)} \leftarrow$ initializes a subset in $\Omega=\big\{\mathcal{S}^{(0)}\subseteq\mathcal{M}_{L_c} \,\big|\, |\mathcal{S}^{(0)}|=L_s\big\}$;
    \FOR{$t=0$ to $N_{\mathrm{step}}-1$}
    \STATE $\mathcal S^{(t)\prime} \leftarrow$ generates a proposal state by a single-element replacement in $\mathcal S^{(t)}$;
    \STATE $\alpha(\mathcal{S}^{(t)}\rightarrow\mathcal{S}^{(t)\prime}) \leftarrow $ computes the MH acceptance probability by Eq.~\eqref{eq:mh_acceptance_prob};
    \STATE $\mathcal S^{(t+1)} \leftarrow$ accepts/rejects $\mathcal S^{(t)\prime}$ according to $u\sim\mathcal U(0,1)$ and $\alpha(\mathcal{S}^{(t)}\rightarrow\mathcal{S}^{(t)\prime})$;
\ENDFOR
    \STATE $\mathcal M_{L_s} \leftarrow $ outputs the selected subset;
    \STATE \textbf{Stage III: CLR --- In-Context Likelihood Ranking}
    \STATE $(\tilde{\mathbf s}^{(j)},\ell^{(j)}) \leftarrow $ decodes and computes likelihood with contextual information by Eq.~\eqref{eq:llm_decode};
    \STATE $\tilde{\mathbf s}^{\ast} \leftarrow $ selects the final reconstruction by Eq.~\eqref{eq:rlfd_explicit}.
\end{algorithmic}
\end{algorithm}

\section{Simulation settings and results}\label{sec4_Experiment}
\subsection{Simulation Settings}

\begin{table}[!t]
\caption{Main hyperparameter settings used in the experiments.}
\label{tab2:settings}
\centering
\begin{tabular*}{0.44\textwidth}{clc}
\toprule
\textbf{Model} & \textbf{Hyperparameter} & \textbf{Value} \\
\midrule

\multirow{5}{*}{ECCT}
& Learning rate & $10^{-4}$ \\
& Batch size & 128 \\
& Number of decoder layers & 6 \\
& Embedding dimension & 32 \\
& Number of attention heads & 8 \\
\midrule

\multirow{7}{*}{DeepSC}
& Learning rate & $10^{-4}$ \\
& Batch size & 64 \\
& Number of encoder/decoder layers & 4 \\
& Embedding dimension & 128 \\
& FFN hidden dimension & 512 \\
& Number of attention heads & 8 \\
& Channel dimension & 16 \\
\midrule

\multirow{7}{*}{UT}
& Learning rate & $10^{-4}$ \\
& Batch size & 64 \\
& Number of encoder/decoder layers & 3 \\
& Embedding dimension & 128 \\
& FFN hidden dimension\footnotemark[3] & 1024 \\
& Number of attention heads & 8 \\
& Channel dimension & 16 \\
\midrule

\multirow{7}{*}{BART-JSCC}
& Learning rate & $10^{-4}$ \\
& Batch size & 8 \\
& Number of encoder/decoder layers & 6 \\
& Embedding dimension & 768 \\
& FFN hidden dimension & 3072 \\
& Number of attention heads & 12 \\
& Channel dimension & 16 \\
\bottomrule
\end{tabular*}
\end{table}
\footnotetext[3]{FFN denotes the Feed-Forward Network.}

We evaluate the proposed receiver-side decoding framework ICD against conventional SSCC pipelines~\cite{SSCC1} and representative JSCC-based semantic communication schemes \cite{DeepSC,UT,UT_quanti} over both AWGN and Rayleigh fading channels. To ensure a fair and reproducible comparison, all methods are tested on the same text source and are assessed using widely adopted NLP-oriented quality metrics \cite{BLEU,BERT}.

\textbf{Dataset and source coding.} 
The adopted dataset is the proceedings of the European Parliament~\cite{dataset}, which consists of around $2.0$ million sentences. The dataset is pre-processed to have sentence lengths ranging from $4$ to $30$ words for training and testing. Unless otherwise stated, we adopt GPT-$2$ base ($124.44$M parameters) as the default source model. The arithmetic coder is configured with a precision of $31$ bits, which provides a practical trade-off between numerical stability and coding accuracy. For each corpus, some leading words, regarded as the contextual information $\textbf{m}_{\text{pre}}$, have been correctly decoded. In the current experiments, $\textbf{m}_{\text{pre}}$ corresponds to approximately the leading half of the source text, i.e., $\eta_{\text{raw}}\approx 0.5$. 

\textbf{Channel coding and ECCT configuration.} 
For the SSCC baselines, we employ an LDPC($49, 24$) code, corresponding to a code rate of approximately $0.5$. The ECCT architecture and training hyperparameters, as well as those of the JSCC baselines, are summarized in Table~\ref{tab2:settings}.

\textbf{Baselines.} 
The baselines are selected to cover representative text transmission paradigms, including conventional entropy-coding-based SSCC, ECCT-assisted LLM-based SSCC, and JSCC-based semantic communication.
\begin{itemize}
    \item Huffman with ECCT, where the LLM-driven arithmetic source coding/decoding is replaced by Huffman coding, while keeping the rest of the SSCC pipeline unchanged.
    \item LLM-AC with ECCT, which combines ECCT-assisted channel decoding with LLM-based lossless source decoding under the SSCC pipeline.
    \item DeepSC~\cite{DeepSC} and UT~\cite{UT}, which serve as representative JSCC baselines for comparison.
    \item BART~\cite{bart}-JSCC, which employs a pretrained BART encoder-decoder backbone and trainable channel mapping modules for end-to-end semantic transmission.
    \item UT with quantization~\cite{UT_quanti}, where the continuous latent representation is mapped to a fixed-length bitstream ($30$ bits) for transmission.
\end{itemize}

\textbf{Evaluation metrics.} 
Reconstruction quality is quantified by Bilingual Evaluation Understudy (BLEU)~\cite{BLEU} and a Bidirectional Encoder Representations from Transformers (BERT)-based semantic similarity metric~\cite{BERT}, which are widely used in text reconstruction tasks to assess lexical fidelity and semantic preservation, respectively. 

On the other hand, when comparing heterogeneous pipelines (e.g., SSCC vs.\ JSCC), different schemes generally transmit different amounts of physical-layer payload $N$ for the same source content. Therefore, to maintain consistent energy accounting, consistent with \cite{SSCC1}, 
we normalize SNR using a fixed total transmission energy budget. Let $E_{\mathrm{total}}$ denote the total energy consumed by sending $N_{\mathrm{um}}$ payload units through the channel in an LLM-based SSCC reference system~\cite{SSCC1}, and define the unified SNR according to
\begin{equation}
\begin{aligned}
\mathrm{SNR}
&=
10\log\!\left(\frac{E_{\mathrm{total}}}{N_0\,N}\right)  \\
&=
10\log\!\left(\frac{E_{\mathrm{total}}}{N_0\,N_{\mathrm{um}}}\right)
+
10\log\!\left(\frac{N_{\mathrm{um}}}{N}\right) \\
&=
\mathrm{SNR}_{\mathrm{unified}}
+10\log\!\left(\frac{N_{\mathrm{um}}}{N}\right).
\label{eq:unified_snr_bits}
\end{aligned}
\end{equation}
Here, $\mathrm{SNR}_{\mathrm{unified}}$ serves as the independent variable that aligns $E_{\mathrm{total}}$ across different methodologies, while the ratio $N_{\mathrm{um}}/N$ compensates for the differing physical-layer payload lengths. Meanwhile, for JSCC methods such as DeepSC~\cite{DeepSC} and UT~\cite{UT}, the transmitter commonly outputs continuous latent vectors, and the channel input is represented by floating-point values rather than binary bits. To maintain consistent energy accounting, we treat one transmitted float as consuming the energy of multiple bits. With float16, one float corresponds to 16 bits, resulting in an additional offset $10\log(16)\approx 12.041$ dB.
Accordingly, the unified evaluation metric is given by Eq.~\eqref{eq:unified_snr}.
\begin{figure*}[!t]
\centering
\begin{equation}
\label{eq:unified_snr}
\mathrm{SNR}=
\begin{cases}
\mathrm{SNR}_{\mathrm{unified}}
+10\log\!\left(\frac{N_{\mathrm{um}}}{N}\right)
+12.041, & \text{float-based JSCC (float16)},\\[4pt]
\mathrm{SNR}_{\mathrm{unified}}
+10\log\!\left(\frac{N_{\mathrm{um}}}{N}\right), & \text{bit-oriented transmission}.
\end{cases}
\end{equation}
\end{figure*}

\begin{figure*}[!t]
\centering
\includegraphics[width=\linewidth]{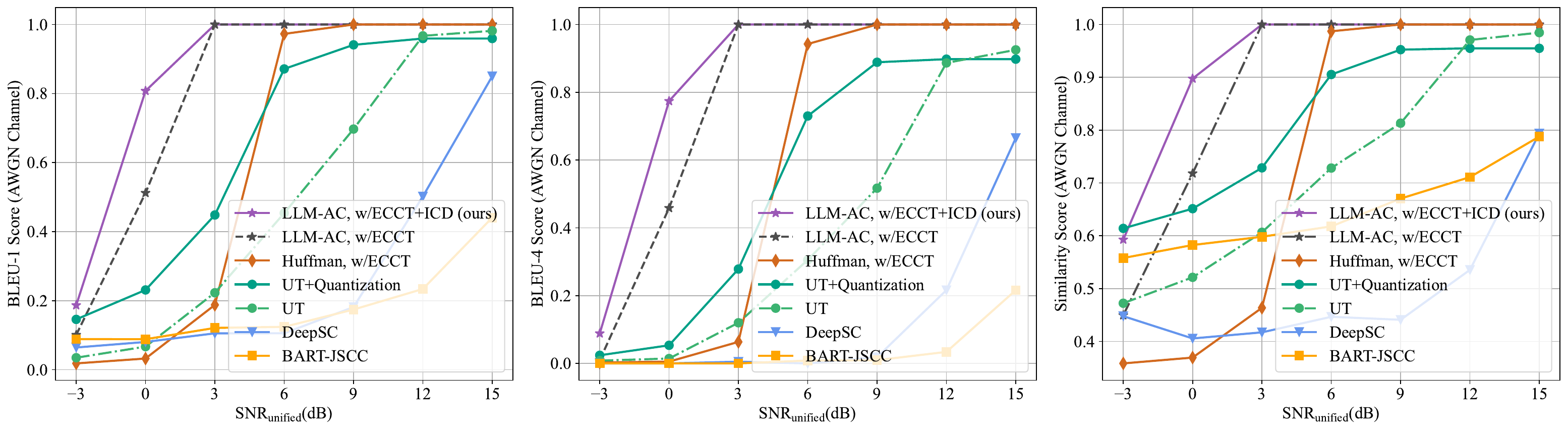}
\caption{Overall system-level performance under AWGN channels.}
\label{fig:overall_performance_awgn}
\end{figure*}

\begin{figure*}[!t]
\centering
\includegraphics[width=\linewidth]{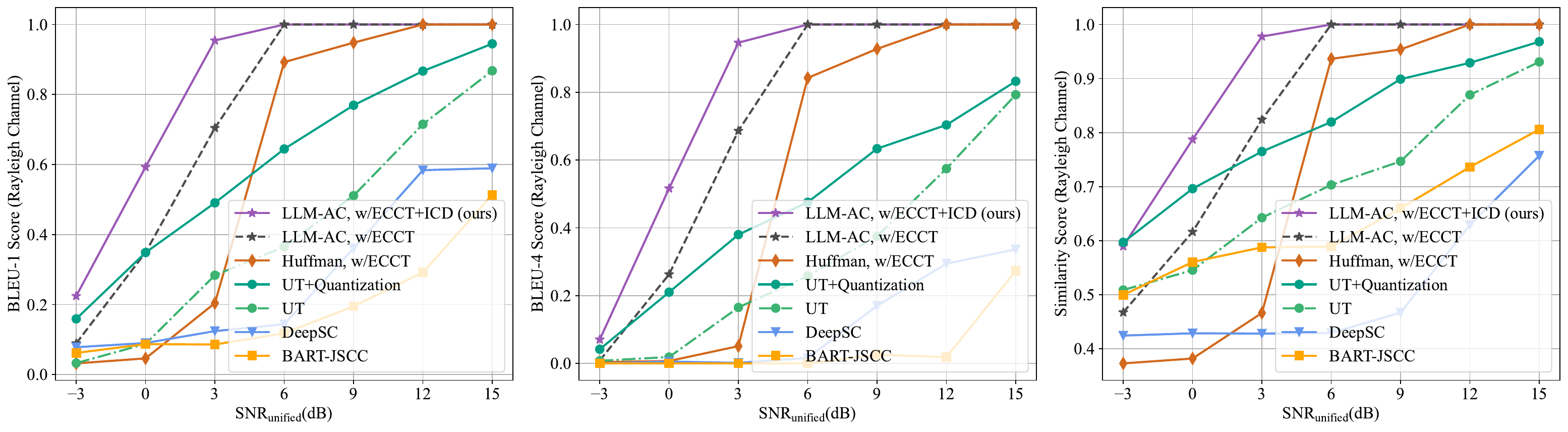}
\caption{Overall system-level performance under Rayleigh channels.}
\label{fig:overall_performance_rayleigh}
\end{figure*}

\begin{figure*}[!t]
\centering
\includegraphics[width=\linewidth]{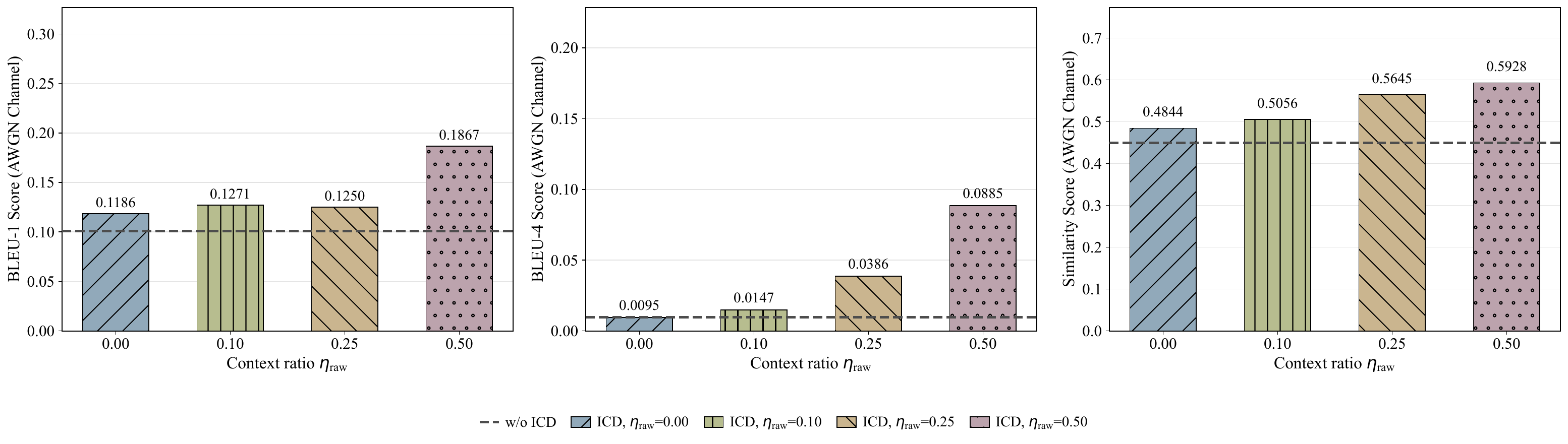}
\vspace{-.2cm}
\caption{Performance of ICD under different context ratios $\eta_{\mathrm{raw}}$ at SNR$_{\mathrm{unified}}=-3$ dB.}
\label{fig:context_ratio}
\end{figure*}

\begin{figure*}[!t]
\centering
\includegraphics[width=0.6\linewidth]{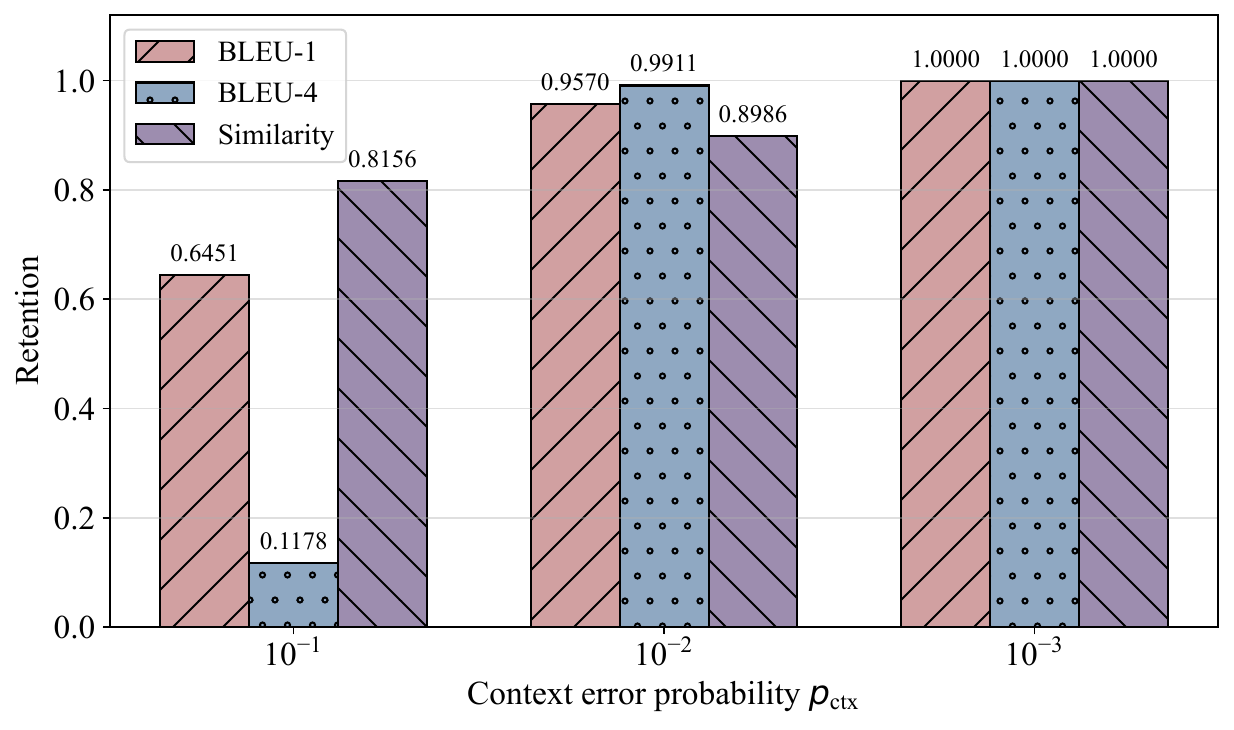}
\caption{Robustness of ICD under different context corruption probabilities at SNR$_{\mathrm{unified}}=-3$ dB.}
\label{fig:robust}
\end{figure*}

\begin{table*}[!t]
\caption{Computational complexity and inference latency comparison.}
\label{tab:complexity_latency}
\tabcolsep 3.2pt
{
\begin{tabular*}{\textwidth}{@{\extracolsep{\fill}}lccccccccc}
\toprule
\multirow{2}{*}{Metric} 
& \multicolumn{3}{c}{LLM-AC, w/ECCT+ICD} 
& \multirow{2}{*}{\makecell{LLM-AC,\\ w/ECCT}} 
& \multirow{2}{*}{\makecell{Huffman,\\ w/ECCT}} 
& \multirow{2}{*}{\makecell{UT+\\Quantization}} 
& \multirow{2}{*}{UT} 
& \multirow{2}{*}{DeepSC}
& \multirow{2}{*}{BART-JSCC}\\
\cmidrule(lr){2-4}
& $L_s=5$ 
& $L_s=10$ 
& $L_s=15$ 
&  &  &  &  &  \\
\midrule
FLOPs
& 2695.31G & 4957.62G & 7177.6049G & 942.69G & 3.50G & 12.03G & 12.03G & 57.04G & 1287.69G \\

Total latency (s)
& 144.06 & 263.22 & 400.77 & 45.00 & 45.00 & 16.26 & 17.54 & 4.39 &21.07\\
\bottomrule
\end{tabular*}
}
\end{table*}


\subsection{Overall System-Level Performance Comparison}

Fig.~\ref{fig:overall_performance_awgn} and Fig.~\ref{fig:overall_performance_rayleigh} present the system-level overall performance comparison between the proposed method and multiple baselines over AWGN and Rayleigh channels, respectively, evaluated by BLEU-1, BLEU-4, and semantic similarity. The results demonstrate that the proposed receiver-side ICD consistently achieves the best performance in both channel environments, outperforming all baselines. The performance improvement is most evident in the low-SNR regime, where residual bit errors after channel decoding can trigger interval misselection in LLM-driven arithmetic decoding, resulting in a pronounced performance cliff. In this regard, ICD leverages the contextual information to stabilize the initial interval localization and prevent catastrophic early deviation, then constructs multiple bitstream candidates via reliability-guided bit flipping, while a sampling step further enhances candidate diversity and controls decoding overhead.
Notably, although BART-JSCC benefits from pretrained semantic modeling, its relatively long continuous latent representation and aggressive compression from the $768$-dimensional hidden space to a $16$-dimensional channel representation introduce additional energy and representation bottlenecks, limiting its performance under the unified-energy setting.
We further investigate the impact of context length in Fig.~\ref{fig:context_ratio}. The results show that ICD remains effective even without contextual information, provided that the other proposed components (CCG, CCS and CLR) are retained. Furthermore, increasing the context ratio further improves reconstruction performance, with the best result among the evaluated settings achieved at $\eta_{\mathrm{raw}}=0.5$. This is because a longer context provides a stronger conditional prior for AC interval localization and shortens the error-sensitive suffix, thereby reducing the propagation of residual decoding errors.
We also report the computational complexity and inference latency of different schemes in Table~\ref{tab:complexity_latency} to assess the practical feasibility of ICD. The results show that ICD introduces additional receiver-side computation compared with ECCT-only SSCC and end-to-end JSCC baselines. This cost increases with the sampled subset size $L_s$. Therefore, $L_s$ explicitly controls the computational cost of ICD by determining the number of LLM-based decoding and likelihood-evaluation passes. This provides a tunable robustness-complexity trade-off: smaller $L_s$ can be used under stricter latency constraints, while larger $L_s$ can be adopted when stronger decoding robustness is required. To evaluate the robustness against inaccurate context, we sweep the context corruption probability $p_{\mathrm{ctx}}\in\{10^{-1},10^{-2},10^{-3}\}$, with a substantially more adverse setting than typical URLLC-level reliability requirements~\cite{3gpp}. Fig.~\ref{fig:robust} illustrates the normalized performance retention rate relative to the error-free context case. The results show that noticeable degradation occurs only under the extreme setting of $p_{\mathrm{ctx}}=10^{-1}$ at SNR$_{\mathrm{unified}}=-3$ dB, while all metrics rapidly recover to near-saturation levels when $p_{\mathrm{ctx}}\leq 10^{-2}$, indicating the robustness of ICD under deliberately severe context-corruption conditions.

\subsection{Impact of Confidence-Ranked and Sampled Candidates Hyperparameters}

Fig.~\ref{fig:heatmap_awgn} and Fig.~\ref{fig:heatmap_rayleigh} show the BLEU-4 performance of ICD under two coupled hyperparameters, namely the size of the confidence-ranked candidate pool and the number of sampled candidates, which determines how many candidate bitstreams are ultimately decoded. A consistent trend is that using a moderate confidence-ranked pool together with a moderate number of sampled candidates yields the most reliable gains across SNR points. Further increasing the number of sampled candidates does not reliably yield better performance, suggesting that ICD is driven primarily by forming high-quality, reliability-consistent candidates, rather than simply decoding as many candidates as possible. Conversely, using too few sampled candidates under-explores the candidate space and makes ICD more prone to missing the correct interval updating path at low SNR. A similar non-monotonic effect is observed for the confidence-ranked pool size. When the ranked pool is reduced, ICD may discard the true bitstream prematurely, biasing the sampled set toward a narrow subset. By contrast, a larger ranked pool may include additional low-confidence candidates, potentially reducing sampling focus and weakening the reliability–likelihood fusion through less effective exploration. These results also reflect the performance-complexity trade-off of ICD. A larger $L_c$ increases the subset search and energy-evaluation overhead, while a larger $L_s$ almost linearly increases the number of LLM-based decoding attempts. Since moderate $(L_c,L_s)$ settings achieve the best BLEU-4 performance in Fig.~\ref{fig:heatmap_awgn} and Fig.~\ref{fig:heatmap_rayleigh}, ICD does not require exhaustive candidate decoding and can be configured to balance robustness and latency in real-time applications.

\begin{figure*}[!t]
\centering
\includegraphics[width=\linewidth]{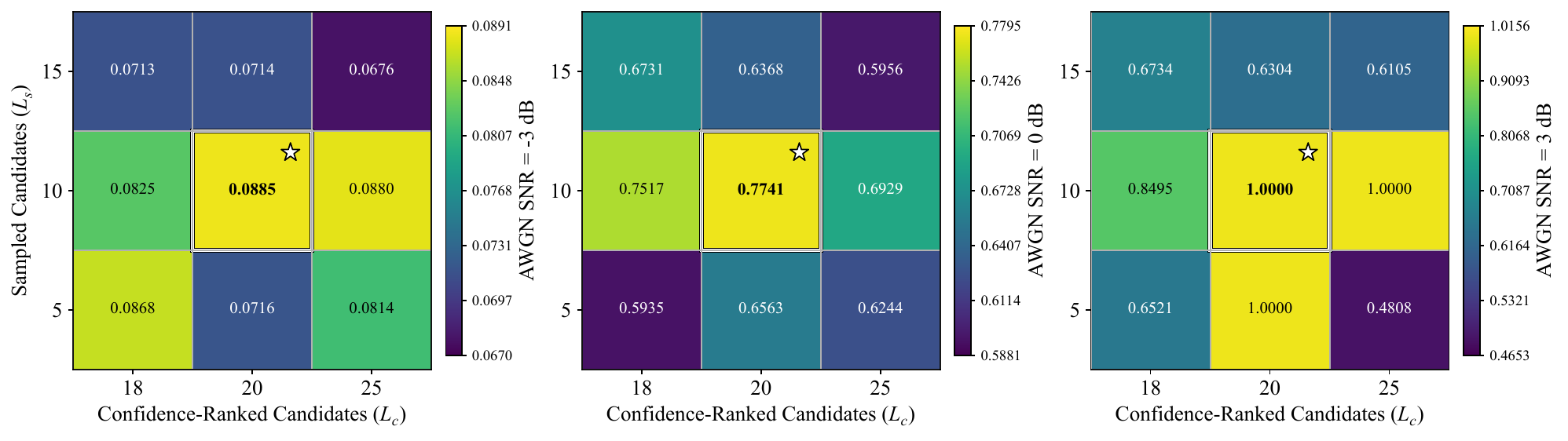}
\caption{BLEU-4 heatmaps of ICD performance under AWGN channels.}
\label{fig:heatmap_awgn}
\end{figure*}

\begin{figure*}[!t]
\centering
\includegraphics[width=\linewidth]{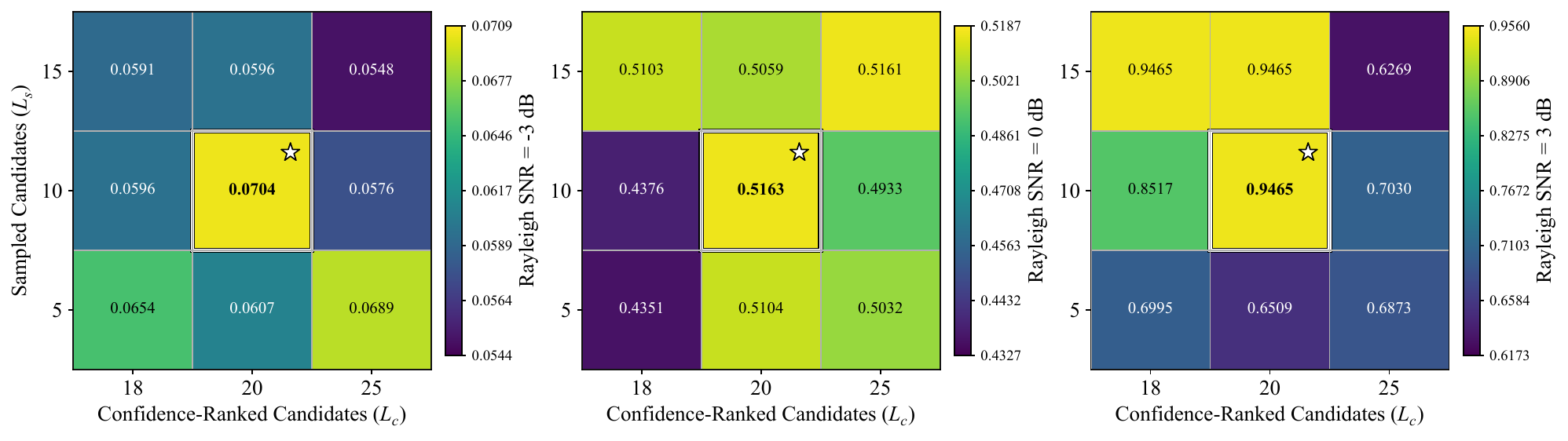}
\caption{BLEU-4 heatmaps of ICD performance under Rayleigh channels.}
\label{fig:heatmap_rayleigh}
\end{figure*}

\begin{table}[t]
\centering
\caption{Model sizes of the pretrained language-model backbones.}
\label{tab:llm_backbones}
\begin{tabular*}{0.24\textwidth}{cc}
\toprule
Model & Parameters (M) \\
\midrule
GPT-Base   & 124.44 \\
GPT-Medium & 379.99 \\
GPT-Large  & 811.78 \\
GPT-XL     & 1607.94 \\
BART    & 139.42 \\
\bottomrule
\end{tabular*}
\end{table}

\subsection{Scalability of the LLM Backbone}

We examine how the LLM backbone affects the performance of the proposed ICD framework. Specifically, we evaluate ICD with progressively larger size of LLMs while keeping all other system components unchanged, and the parameter counts of the evaluated backbones are summarized in Table~\ref{tab:llm_backbones}. As shown in Fig.~\ref{fig:GPT_scales}, consistent but moderate gains are observed as the backbone size increases. This trend indicates that ICD can directly benefit from improved language modeling quality, as larger models provide more accurate conditional probability estimates for candidate generation and refinement. At the same time, the performance gap between different model scales remains bounded, suggesting that ICD is largely model-agnostic and already achieves strong performance with the base model. These results confirm that the proposed framework scales favorably with model capacity while maintaining robustness across different backbone choices.

\begin{figure*}[!t]
\centering
\includegraphics[width=\linewidth]{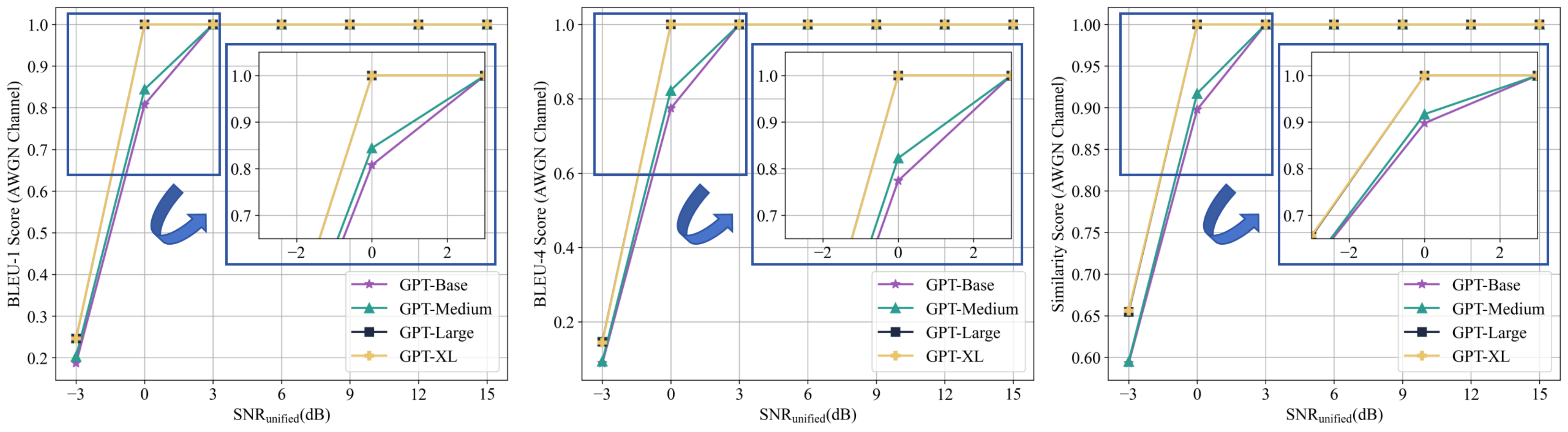}
\caption{Scalability evaluation with varying LLM backbone sizes under AWGN channels.}
\label{fig:GPT_scales}
\end{figure*}

\begin{table*}[!t]
\caption{Results under AWGN channels with different code rates.}
\label{tab:ablation_coderate_awgn}
\centering
\begin{tabular*}{\textwidth}{@{\extracolsep{\fill}}
>{\centering\arraybackslash}p{0.12\textwidth}
>{\centering\arraybackslash}p{0.08\textwidth}
p{0.3\textwidth}
>{\centering\arraybackslash}p{0.10\textwidth}
>{\centering\arraybackslash}p{0.10\textwidth}
>{\centering\arraybackslash}p{0.14\textwidth}}
\toprule
\multirow{2}{*}{Code rate $(N,K)$} & \multirow{2}{*}{SNR (dB)} & \multirow{2}{*}{Method} & \multicolumn{3}{c}{Metrics} \\
\cmidrule(lr){4-6}
& & & BLEU1 & BLEU4 & Similarity Score \\
\midrule

\multirow{9}{*}{(49,24)}
& \multirow{3}{*}{-3}  & LLM-AC, w/ECCT+ICD  & \textbf{0.1867} & \textbf{0.0885} & \textbf{0.5928} \\
&                           & LLM-AC, w/ECCT+Context & 0.1832 & 0.0714 & 0.5921 \\
&                           & LLM-AC, w/ECCT & 0.1007 & 0.0097 & 0.4496 \\
\cmidrule(lr){2-6}
& \multirow{3}{*}{0}  & LLM-AC, w/ECCT+ICD  & \textbf{0.8076} & \textbf{0.7741} & \textbf{0.8976} \\
&                           & LLM-AC, w/ECCT+Context & 0.6647 & 0.6135 & 0.8398 \\
&                           & LLM-AC, w/ECCT & 0.5120 & 0.4584 & 0.7181 \\
\cmidrule(lr){2-6}
& \multirow{3}{*}{3}  & LLM-AC, w/ECCT+ICD  & 1.0000 & 1.0000 & 1.0000 \\
&                           & LLM-AC, w/ECCT+Context & 1.0000 & 1.0000 & 1.0000 \\
&                           & LLM-AC, w/ECCT & 1.0000 & 1.0000 & 1.0000 \\
\midrule

\multirow{9}{*}{(49,30)}
& \multirow{3}{*}{-3}  & LLM-AC, w/ECCT+ICD  & \textbf{0.1778} & \textbf{0.0708} & \textbf{0.5462} \\
&                           & LLM-AC, w/ECCT+Context & 0.1398 & 0.0530 & 0.5275 \\
&                           & LLM-AC, w/ECCT & 0.0937 & 0.0093 & 0.4623 \\
\cmidrule(lr){2-6}
& \multirow{3}{*}{0}  & LLM-AC, w/ECCT+ICD  & \textbf{0.5457} & \textbf{0.4833} & \textbf{0.7532} \\
&                           & LLM-AC, w/ECCT+Context & 0.5125 & 0.4429 & 0.7376 \\
&                           & LLM-AC, w/ECCT & 0.3152 & 0.2322 & 0.6313 \\
\cmidrule(lr){2-6}
& \multirow{3}{*}{3}  & LLM-AC, w/ECCT+ICD  & \textbf{0.9697} & \textbf{0.9602} & \textbf{0.9852} \\
&                           & LLM-AC, w/ECCT+Context & \textbf{0.9697} & \textbf{0.9602} & \textbf{0.9852} \\
&                           & LLM-AC, w/ECCT & 0.9444 & 0.9266 & 0.9734 \\
\midrule

\multirow{9}{*}{(49,36)}
& \multirow{3}{*}{-3}  & LLM-AC, w/ECCT+ICD  & 0.1761 & 0.0480 & \textbf{0.5746} \\
&                           & LLM-AC, w/ECCT+Context & \textbf{0.1828} & \textbf{0.0485} & 0.5738 \\
&                           & LLM-AC, w/ECCT & 0.0970 & 0.0090 & 0.4668 \\
\cmidrule(lr){2-6}
& \multirow{3}{*}{0}  & LLM-AC, w/ECCT+ICD  & \textbf{0.4433} & \textbf{0.3602} & \textbf{0.6973} \\
&                           & LLM-AC, w/ECCT+Context & 0.2922 & 0.1719 & 0.5936 \\
&                           & LLM-AC, w/ECCT & 0.1416 & 0.0478 & 0.5133 \\
\cmidrule(lr){2-6}
& \multirow{3}{*}{3}  & LLM-AC, w/ECCT+ICD  & \textbf{0.9762} & \textbf{0.9674} & \textbf{0.9901} \\
&                           & LLM-AC, w/ECCT+Context & 0.8779 & 0.8455 & 0.9402 \\
&                           & LLM-AC, w/ECCT & 0.6786 & 0.6093 & 0.8073 \\
\bottomrule
\end{tabular*}
\end{table*}

\subsection{Ablation Experiments}\label{ablation_experiments}

Table~\ref{tab:ablation_coderate_awgn} presents a comprehensive ablation study under AWGN channels across different code rates and SNRs. Overall, LLM-AC with ECCT+ICD consistently delivers the strongest performance, especially in the low-SNR regimes, where reliable recovery is most challenging. Compared with the ECCT-only baseline, incorporating ICD yields substantial gains across all three metrics, confirming the necessity of structured enhancement beyond ECCT-based decoding. As the code rate increases, performance tends to decline at the same SNR, since higher-rate coding provides less redundancy for error correction and thus leaves more residual errors after decoding, under which the baselines degrade more noticeably while ICD more consistently preserves its advantage.

When the CCS module is removed and only context-assisted decoding is retained (denoted as w/ECCT+Context), performance degrades but remains clearly superior to the ECCT-only baseline. This indicates that context modeling alone provides meaningful improvements, yet is insufficient to fully exploit the candidate space without sampling. The full ICD configuration consistently achieves the best or near-best results, highlighting the complementary roles of context modeling and candidate sampling. We also present a comparison of the runtime test on NVIDIA GeForce RTX 4090 and find that the CCS module yields a $1.6457\times$ speedup. Overall, it validates the effectiveness of the CCS module for a more computationally efficient solution.
\section{Conclusions}\label{sec5_Conclusions}

In this work, we have presented ICD, a receiver-side in-context decoding framework for improving the robustness of SSCC-based transmission in low-SNR regimes. By leveraging ECCT-derived bit-wise reliability and contextual information, ICD has leveraged a three-stage processing chain composed of CCG, CCS, and CLR. Specifically, CCG has produced a confidence-ranked candidate pool via reliability-guided bit flipping, while CCS has performed diversity-preserving sampling to achieve a favorable accuracy–complexity trade-off, and CLR has fused ECCT-derived reliability with source-level log-likelihood to determine the final reconstruction. We have further provided theoretical guarantees on the stability and convergence of this CCS module. Extensive evaluations over AWGN and Rayleigh fading channels have demonstrated consistent improvements compared with conventional SSCC pipelines and representative JSCC schemes, and achieved a favorable balance between decoding reliability and computational efficiency. 
Beyond the text transmission scenario considered in this work, ICD can be extended to other modalities by adapting the source model and contextual information. For image or video transmission, a Large Vision Model (LVM)-based SSCC pipeline can represent visual blocks as pixel-token sequences and use an ImageGPT~\cite{imageGPT}-like autoregressive visual model to provide next-pixel likelihoods for AC-based decoding and candidate ranking. Thus, the linguistic likelihood in CLR can be replaced by visual-token likelihood, while CCG and CCS remain applicable at the bitstream level. In such visual scenarios, contextual information can be instantiated as visual blocks, regions, or frames obtained from previous transmissions and/or reliable auxiliary channels. In future work, we will extend ICD by incorporating stronger model-free decoding beyond ECCT and further investigate modality-specific SSCC designs for reliable visual content transmission.

\bibliographystyle{IEEEtran}
\bibliography{ref}
\end{document}